\documentclass[sigconf]{acmart}
\AtBeginDocument{%
  }

\copyrightyear{2025}
\acmYear{2025}
\setcopyright{acmlicensed}\acmConference[MM '25]{Proceedings of the 33rd ACM International Conference on Multimedia}{October 27--31, 2025}{Dublin, Ireland}
\acmBooktitle{Proceedings of the 33rd ACM International Conference on Multimedia (MM '25), October 27--31, 2025, Dublin, Ireland}
\acmDOI{10.1145/3746027.3755604}
\acmISBN{979-8-4007-2035-2/2025/10}

\usepackage{multirow}
\usepackage{amsmath}
\usepackage{algorithm}
\usepackage{algorithmic}
\usepackage{makecell}
\usepackage{subfigure}
\usepackage{enumitem}
\usepackage{hyperref}

\newcommand{\method}{LEGO}
\DeclareMathOperator*{\argmin}{arg\,min}
\DeclareMathOperator*{\argmax}{arg\,max}
\newtheorem{theorem}{theorem}
\newtheorem{definition}{definition}

\begin{document}

\title{LEGO: A Lightweight and Efficient Multiple-Attribute Unlearning Framework for Recommender Systems}


\author{Fengyuan Yu}
\orcid{0009-0005-0491-1869}
\affiliation{%
  \institution{Zhejiang University}
  \city{Hangzhou}
  \country{China}}
\email{fengyuanyu@zju.edu.cn}

\author{Yuyuan Li}
\orcid{0000-0003-4896-2885}
\affiliation{%
  \institution{Hangzhou Dianzi University}
  \city{Hangzhou}
  \country{China}}
\email{y2li@hdu.edu.cn}

\author{Xiaohua Feng}
\orcid{0009-0001-6829-7088}
\affiliation{%
  \institution{Zhejiang University}
  \city{Hangzhou}
  \country{China}}
\email{fengxiaohua@zju.edu.cn}

\author{Junjie Fang}
\orcid{0009-0002-7981-1609}
\affiliation{%
  \institution{Hangzhou Dianzi University}
  \city{Hangzhou}
  \country{China}}
\email{junjiefang@hdu.edu.cn}

\author{Tao Wang}
\orcid{0009-0008-8649-1301}
\affiliation{%
  \institution{Midea Group}
  \city{Foshan}
  \country{China}}
\email{tao.wang.seu@gmail.com}

\author{Chaochao Chen}
\orcid{0000-0003-1419-964X}
\authornote{Corresponding author}
\affiliation{%
  \institution{Zhejiang University}
  \city{Hangzhou}
  \country{China}}
\email{zjuccc@zju.edu.cn}


\begin{abstract}
With the growing demand for safeguarding sensitive user information in recommender systems, recommendation attribute unlearning is receiving increasing attention.
Existing studies predominantly focus on single-attribute unlearning.
However, privacy protection requirements in the real world often involve multiple sensitive attributes and are dynamic.
Existing single-attribute unlearning methods cannot meet these real-world requirements due to i) \textbf{CH1}: the inability to handle multiple unlearning requests simultaneously, and ii) \textbf{CH2}: the lack of efficient adaptability to dynamic unlearning needs.
To address these challenges, we propose \method, a lightweight and efficient multiple-attribute unlearning framework.
Specifically, we divide the multiple-attribute unlearning process into two steps:
i) \textit{Embedding Calibration} removes information related to a specific attribute from user embedding,
and ii) \textit{Flexible Combination} combines these embeddings into a single embedding, protecting all sensitive attributes.
We frame the unlearning process as a mutual information minimization problem, providing \method{} a theoretical guarantee of simultaneous unlearning, thereby addressing \textbf{CH1}.
With the two-step framework, where \textit{Embedding Calibration} can be performed in parallel and \textit{Flexible Combination} is flexible and efficient, we address \textbf{CH2}.
Extensive experiments on three real-world datasets across three representative recommendation models demonstrate the effectiveness and efficiency of our proposed framework.
Our code and appendix are available at \href{https://github.com/anonymifish/lego-rec-multiple-attribute-unlearning}{https://github.com/anonymifish/lego-rec-multiple-attribute-unlearning}.
\end{abstract}




\maketitle

\section{Introduction}\label{sec:intro}
Modern recommender systems commonly use Collaborative Filtering (CF) algorithms to provide personalized recommendations~\cite{schafer2007collaborative, koren2021advances, liu2025gated, wang2024feddse, wang2023dafkd, li2025personalized, zhou2022advances, wu2025unlearning}.
However, privacy concerns regarding personalized recommendations have increased, with increasing demand for protection against the misuse of sensitive user information.
As a protective measure, the \textit{Right to be Forgotten} requires recommendation platforms to allow users to withdraw individual data~\cite{voigt2017eu, scassa2020data, bonta2022california, chen2022recommendation}.
\textit{Recommendation unlearning} is an emerging approach for addressing these privacy concerns.
One line of research, i.e., \textit{input unlearning}, focuses on enabling the model to forget specific training data~\cite{li2024survey}.
Another line of research, i.e., \textit{attribute unlearning}, focuses on forgetting sensitive user attributes, which are not part of training data and cannot be unlearned through input unlearning~\cite{beigi2020privacy, guo2022efficient, li2023making, feng2025plug}.
While input unlearning has been extensively studied, attribute unlearning remains comparatively underexplored. This paper aims to bridge this gap by focusing on attribute unlearning.

\begin{figure}
    \centering
    \subfigure[Real-world requirements.]{\includegraphics[width=0.52\linewidth]{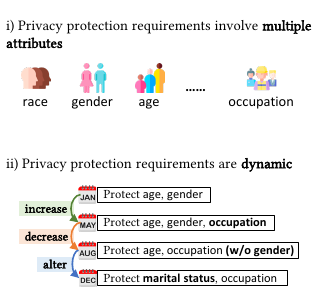} \label{fig:dynamic}}
    \subfigure[Single-attribute unlearning cannot meet dynamic requirements.]{\includegraphics[width=0.45\linewidth]{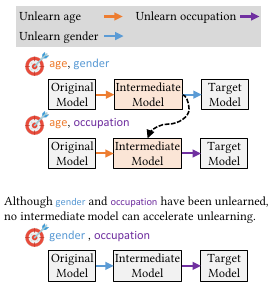} \label{fig:single_gap}}
    \caption{(a) Privacy protection requirements often involve multiple attributes and are dynamic: they may increase, decrease, and alter. (b) Single-attribute unlearning cannot meet dynamic privacy protection requirements. 
    The dashed arrow indicates the storage of the intermediate model can accelerate sequential unlearning.
    }
\end{figure}

Most existing research on attribute unlearning can only handle single and static attributes~\cite{ganhör2022unlearning, li2023making, chen2024posttraining}.
However, in practice, unlearning requests usually involve multiple sensitive attributes and are dynamic: they may increase, decrease, or alter, as illustrated in Figure~\ref{fig:dynamic}.
The frequent changes in privacy protection requirements necessitate attribute unlearning to adapt flexibly and efficiently to these evolving demands.

In this paper, we identify that existing attribute unlearning methods cannot meet these requirements due to two key challenges: \textbf{CH1:} the inability to handle multiple unlearning requests simultaneously, and \textbf{CH2:} the lack of efficient adaptability to dynamic unlearning needs.
Neither i) unlearning each attribute individually using single-attribute unlearning methods (i.e., sequential unlearning) nor ii) the only existing multiple-attribute unlearning method, AdvX~\cite{escobedo2024simultaneous}, can address these two challenges.
For \textbf{CH1}, the sequential unlearning approach may re-introduce previously unlearned attributes into the model while unlearning others, thereby degrading the effectiveness of unlearning.
AdvX, which introduces an adversarial discriminator for each attribute,
faces issues related to potential conflicts in optimization directions,
which results in suboptimal unlearning effectiveness.
For \textbf{CH2}, if the requirements change, the sequential unlearning approach needs to re-apply single-attribute unlearning to each sensitive attribute, even if many of them have already been unlearned.
While saving intermediate models during unlearning alleviates this issue, it consumes considerable memory.
Moreover, in many cases, as shown in Figure~\ref{fig:single_gap}, even with intermediate models, the unlearning process cannot be accelerated.
AdvX is also not adaptable to dynamic privacy protection requirements, as the training process must be re-executed each time the requirement changes.

To address the challenges above, we propose \method{}, a \textbf{L}ightweight and \textbf{E}fficient multiple-attribute unlearnin\textbf{G} Framew\textbf{O}rk.
\method{} divides multiple-attribute unlearning process into two steps: \textit{Embedding Calibration} and \textit{Flexible Combination}.
Firstly, embedding calibration removes information related to a specific attribute from user embedding.
We achieve this by minimizing the mutual information between user embedding and the corresponding attribute.
To preserve recommendation performance, we further introduce a parameter space constraint to ensure that, after calibration, embeddings do not deviate significantly from their original values.
Secondly, flexible combination combines the unlearned embeddings into a single embedding, protecting all sensitive attributes that require protection through a weighted combination.
Only the weights are optimized to ensure an efficient combination.

Our proposed two-step framework effectively addresses both challenges.
Embedding calibration first unlearns a specific attribute, and then flexible combination simultaneously unlearns all attributes by combining these embeddings.
By leveraging the properties of mutual information and the parameter space constraint, we provide a theoretical guarantee for effective simultaneous unlearning of all attributes, addressing \textbf{CH1}.
When a new requirement arises, embedding calibration can be performed in parallel to unlearn attributes not identified in previous requirements, and flexible combination can efficiently construct a new embedding that protects all sensitive attributes, thereby addressing \textbf{CH2}.

We summarize the main contributions of this paper as follows:
\begin{itemize}[leftmargin=*]\setlength{\itemsep}{-\itemsep}
    \item We identify two key challenges of multiple-attribute unlearning in recommender systems (i.e., \textbf{CH1}: handling simultaneous unlearning requirements and \textbf{CH2}: adapting to dynamic needs.).
    To tackle these challenges, we propose a multiple-attribute unlearning framework, named \method{}, which divides the multiple-attribute unlearning process into two steps: \textit{Embedding Calibration} and \textit{Flexible Combination}.
    \item To address \textbf{CH1}, \textit{Embedding Calibration} first unlearns a specific attribute, and then \textit{Flexible Combination} simultaneously unlearns all attributes by combining these embeddings with a theoretical guarantee of effectiveness.
    \item To address \textbf{CH2}, we propose a two-step framework, where \textit{Embedding Calibration} can be performed in parallel to unlearn attributes,
    and \textit{Flexible Combination} can efficiently construct a new embedding that protects all sensitive attributes.
    \item We conduct extensive experiments on three real-world datasets across three representative recommendation models.
    The results demonstrate that our method significantly outperforms existing baselines in terms of multiple-attribute unlearning effectiveness and efficiency.
\end{itemize}

\section{Related Work}\label{sec:related}
In this section, we review two major research lines of recommendation unlearning: traditional recommendation unlearning (input unlearning) and recommendation attribute unlearning.

\subsection{Recommendation Unlearning}
Machine unlearning aims to remove the influence of specific training data on a learned model (i.e., input unlearning)~\cite{nguyen2022survey}.
Existing machine unlearning methods can be categorized into two main approaches:
i) Exact unlearning aims to remove the target data's influence as completely as if the model were retrained from scratch~\cite{cao2015towards, bourtoule2021machine}.
ii) Approximate unlearning aims to estimate the influence of the target data and directly removes the influence through parameter manipulation~\cite{golatkar2020eternal, guo2019certified, sekhari2021remember, warnecke2023machine}.

Following the partition-aggregation framework proposed by SISA (exact unlearning)~\cite{bourtoule2021machine}, subsequent studies achieve exact unlearning tailored for recommender systems~\cite{chen2022recommendation, li2024making, li2023ultrare}.
%
Approximate unlearning has also been explored in the context of recommendation~\cite{li2023selective, zhang2025recommendation}.
A benchmark has been proposed to comprehensively evaluate various recommendation unlearning methods~\cite{chen2024curerec}.

\subsection{Recommendation Attribute Unlearning}
Due to the information extraction capabilities of recommender systems, sensitive attributes such as gender, race, and age of users can be encoded into user embeddings.
However, since these attributes are not explicitly represented in the training data, input unlearning (even exact unlearning or retraining from scratch) cannot effectively address attribute unlearning.

Existing research on recommendation attribute unlearning predominately focuses on single-attribute unlearning.
\citet{ganhör2022unlearning} is the first to address the attribute unlearning problem in recommender systems.
They employ adversarial training during model training on a VAE-based recommendation model, MultVAE~\cite{liang2018variational}, to achieve attribute unlearning.
\citet{li2023making} explore post-training attribute unlearning by directly manipulating model parameters after the training process.
%
%
This work focuses on the attributes with binary labels; in a later work, \citet{chen2024posttraining} extend the method to handle multiple-label attributes.
The only work addressing multiple-attribtue unlearning, AdvX~\cite{escobedo2024simultaneous}, extends the approach of Adv~\cite{ganhör2022unlearning} by introducing an additional attack discriminator for each attribute. 
However, these methods fail to meet real-world dynamic privacy protection requirements due to two key challenges: i) the inability to handle multiple unlearning requests simultaneously and ii) the lack of efficient adaptability to dynamic unlearning needs.
%

\section{Preliminaries}\label{sec:preliminaries}

\subsection{Recommendation Model}
Among recommendation models, CF is a well-established algorithm for generating personalized recommendations by analyzing collaborative information between users and items~\cite{shi2014collaborative}.
Let $\mathcal{U} = \left\{ u_1, \dots u_N\right\}$ and $\mathcal{V} = \left\{v_1, \dots v_M\right\}$ denote the user and item set, respectively.
%
%
%
In general, many existing CF approaches optimize users' latent representations, a.k.a., user embedding, during training to generate personalized recommendations.
We denote user embedding of the model as $[\boldsymbol{\theta}_{1}^{\top}, \dots, \boldsymbol{\theta}_{N}^{\top}] = \boldsymbol{U} \in \mathbb{R}^{N \times d}$, where $\boldsymbol{\theta}_{i} \in \mathbb{R}^{d}$ represents the transpose of the embedding of user $u_{i}$ ($d$ is the dimension of latent space).
We denote the set of attributes as $\mathcal{A} = \{\boldsymbol{A}_{1}, \boldsymbol{A}_{2}, \dots \}$, where $\boldsymbol{A}_{i} = \{c^{i}_{1}, \dots, c^{i}_{p_i} \}$ represents a sensitive attribute, and each $c^{i}_{j}$ denotes a possible value of attribute $\boldsymbol{A}_{i}$ .
We denote the value of attribute $i$ for user $u_{j}$ as $a^{i}_{j}$, $a^{i}_{j} \in \boldsymbol{A}_{i}$.

\subsection{Attacking Setting}
Following the settings in the previous research~\cite{li2023making, chen2024posttraining, wu2020joint, zhang2021graph}, the attack process in the attribute unlearning problem of recommender systems is also referred to as the Attribute Inference Attack (AIA)~\cite{beigi2020privacy, jia2018attriguard}, which is divided into three main stages: exposure, training, and attack.
We adopt the assumption of a gray-box attack during the exposure stage, meaning that not all model parameters are exposed to the attacker; only the embeddings of users and some of their associated attribute information are revealed.
In the training stage, it is assumed that the attacker trains the attack model on the shadow dataset~\cite{salem2018ml}, as assuming the attacker possesses the entire dataset is overly idealistic and impractical.
In the context of multiple attribute unlearning, we assume that during the training stage, the attacker trains a separate attack model for each sensitive attribute.
The attack process is framed as a classification task, where the attack model takes users' embedding as input and the attributes as labels.
In the inference phase, the attacker utilizes their attack model to make predictions.

\subsection{Mutual Information Estimation}
In our framework, we employ Mutual Information (MI) minimization to achieve attribute unlearning because there is a natural link between MI and classification accuracy~\cite{meyen2016relation, cover1999elements, zhang2024cf, liu2025setransformer}.
MI $I(\boldsymbol{x}; \boldsymbol{y})$ is a fundamental measure of the dependence between two random variables, which represents the reduction in the uncertainty of $\boldsymbol{x}$ due to the knowledge of $\boldsymbol{y}$.
If the MI between user embedding and the sensitive attribute is zero, the embedding carries no useful information for predicting the attribute.
In this case, the optimal classifier would be one that randomly guesses the attribute based on its distribution in the sample.

Mathematically, the definition of MI between variables $\boldsymbol{x}$ and $\boldsymbol{y}$ is the relative entropy between the joint distribution and the product distribution $p(\boldsymbol{x})p(\boldsymbol{y})$:
\begin{equation}
    I(\boldsymbol{x}; \boldsymbol{y}) = \mathbb{E}_{p(\boldsymbol{x}, \boldsymbol{y})} \left[\log \frac{p(\boldsymbol{x}, \boldsymbol{y})}{p(\boldsymbol{x}) p(\boldsymbol{y})}\right].
\end{equation}
Calculating the exact value of MI is challenging, as it requires closed-form expression for the density functions and a tractable log-density ratio between the joint and marginal distributions~\cite{belghazi2018mutual, wang2025evaluating}.
To estimate MI, previous work~\cite{cheng2020club} derives CLUB, a contrastive log-ratio upper bound for MI.
With the conditional distribution $p(\boldsymbol{y} \mid \boldsymbol{x})$, MI contrastive log-ratio upper bound is defined as:
\begin{align}
\begin{split}
    {I}_\text{CLUB}(\boldsymbol{x} ; \boldsymbol{y}) & = \mathbb{E}_{p(\boldsymbol{x} , \boldsymbol{y})} \left[\log p(\boldsymbol{y} \mid \boldsymbol{x})\right] \\
    & - \mathbb{E}_{p(\boldsymbol{x})} \mathbb{E}_{p(\boldsymbol{y})} \left[\log p(\boldsymbol{y} \mid \boldsymbol{x})\right].
\end{split}
\end{align}
When the conditional distributions $p(\boldsymbol{y} \mid \boldsymbol{x})$ or $p(\boldsymbol{x} \mid \boldsymbol{y})$ are unavailable, CLUB uses a variational distribution $q_{\phi} (\boldsymbol{y} \mid \boldsymbol{x})$ with parameter $\phi$ to approximate $p(\boldsymbol{y} \mid \boldsymbol{x})$.
A variational CLUB term (vCLUB) is defined as follows:
\begin{align}
\begin{split}
    {I}_\text{vCLUB}(\boldsymbol{x} ; \boldsymbol{y}) & = \mathbb{E}_{p(\boldsymbol{x} , \boldsymbol{y})} \left[\log q_{\phi} (\boldsymbol{y} \mid \boldsymbol{x})\right] \\
    & - \mathbb{E}_{p(\boldsymbol{x})} \mathbb{E}_{p(\boldsymbol{y})} \left[\log q_{\phi} (\boldsymbol{y} \mid \boldsymbol{x})\right].
\end{split}
\end{align}
vCLUB no longer guarantees an upper bound of $I(\boldsymbol{x}; \boldsymbol{y})$ using the variational approximation $q_{\phi} (\boldsymbol{y} \mid \boldsymbol{x})$.
However, with a good variational approximation $q_{\phi} (\boldsymbol{y} \mid \boldsymbol{x})$, vCLUB can still hold an upper bound on MI.
Denote $q_{\phi} (\boldsymbol{x}, \boldsymbol{y}) = q_{\phi} (\boldsymbol{y} \mid \boldsymbol{x}) p(\boldsymbol{x})$, CLUB proves that vCLUB remains a MI upper bound if
\begin{equation}
    KL \left( p(\boldsymbol{x}, \boldsymbol{y}) \Vert q_{\phi} (\boldsymbol{x}, \boldsymbol{y}) \right) \le KL \left( p(\boldsymbol{x}) p(\boldsymbol{y}) \Vert q_{\phi} (\boldsymbol{x}, \boldsymbol{y}) \right).
\end{equation}
This inequality suggests that vCLUB remains a MI upper bound if the variational joint distribution $q_{\phi} (\boldsymbol{x}, \boldsymbol{y})$ is "closer" to $p(\boldsymbol{x}, \boldsymbol{y})$ than to $p(\boldsymbol{x}) p(\boldsymbol{y})$.
Therefore, minimizing $KL ( p(\boldsymbol{x}, \boldsymbol{y}) \Vert q_{\phi} (\boldsymbol{x}, \boldsymbol{y}) )$ helps satisfy the condition for vCLUB to remain an upper bound on MI.
This KL divergence can be minimized by maximizing the log-likelihood of $q_{\phi} (\boldsymbol{y} \mid \boldsymbol{x})$, because of the following equation:
\begin{align}
& \min_{\phi} KL \left( p(\boldsymbol{x}, \boldsymbol{y}) \Vert q_{\phi} (\boldsymbol{x}, \boldsymbol{y}) \right) \notag \\
 = & \min_{\phi} \mathbb{E}_{p(\boldsymbol{x}, \boldsymbol{y})} \left[ \log \left( p(\boldsymbol{y} \mid \boldsymbol{x}) p(\boldsymbol{x}) \right) - \log \left( q_{\phi}(\boldsymbol{y} \mid \boldsymbol{x}) p(\boldsymbol{x}) \right) \right] \notag \\
 = & \min_{\phi} \mathbb{E}_{p(\boldsymbol{x}, \boldsymbol{y})} \left[ \log p(\boldsymbol{y} \mid \boldsymbol{x}) \right] - \mathbb{E}_{p(\boldsymbol{x}, \boldsymbol{y})} \left[ \log q_{\phi}(\boldsymbol{y} \mid \boldsymbol{x}) \right].
 \label{eq:club-opt}
\end{align}
The first term of Eq.~\eqref{eq:club-opt} is independent of the parameter $\phi$.
Therefore, this minimization problem is equivalent to maximizing the second term.
Thus, given samples $\{ ( \boldsymbol{x}_{i}, \boldsymbol{y}_{i} ) \}_{i=1}^{B}$, maximizing the log-likelihood function
\begin{equation}
    \mathcal{L}(\phi) = \frac{1}{B} \sum_{i=1}^{B} \log q_{\phi} (\boldsymbol{y}_{i} \mid \boldsymbol{x}_{i}),
    \label{eq:log-likelihood}
\end{equation}
which leads to a better variational approximation.

\begin{figure*}[t]
    \centering
    \includegraphics[width=\linewidth]{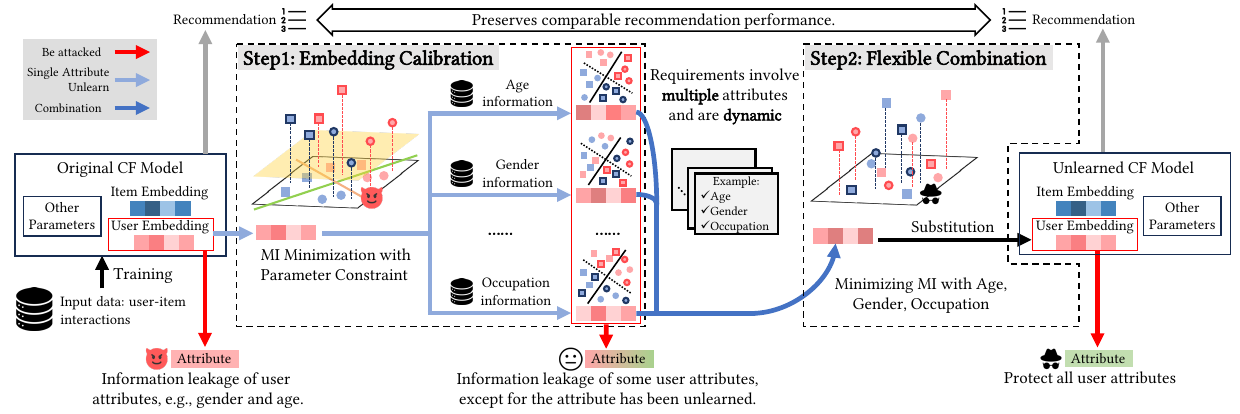}
    \caption{An overview of \method{}. Our proposed \method{} splits multiple-attribute unlearning into two steps: embedding calibration and flexible combination. To illustrate the goal of each step, we provide a sketch of the embedding distribution. In the sketch, the shape, color, and border of the data points represent age, gender, and occupation information, respectively. The lines and plane represent the decision boundaries of the classifier.}
    \label{fig:pipline}
\end{figure*}

In general, MI minimization aims to reduce the correlation between two variables $\boldsymbol{x}$ and $\boldsymbol{y}$ by selecting an optimal parameter $\sigma$ if the joint variational distribution $p_{\sigma} (\boldsymbol{x}, \boldsymbol{y})$.
With vCLUB, MI can be minimized through an alternative optimization approach.
In each training iteration, vCLUB first optimizes $\phi$ by maximizing the log-likelihood $\mathcal{L}(\phi)$ with sampled data points to obtain a better variational approximation.
Then, it estimates the upper bound of MI as follows:
\begin{align}
\hat{I}_{\text{vCLUB}} & = \frac{1}{B^2} \sum_{i=1}^{B} \sum_{j=1}^{B} \left[ \log q_{\phi} (\boldsymbol{y}_i \mid \boldsymbol{x}_i) - \log q_{\phi} (\boldsymbol{y}_j \mid \boldsymbol{x}_i)\right] \notag \\
 & = \frac{1}{B} \sum_{i=1}^{B} \left[ \log q_{\phi} (\boldsymbol{y}_i \mid \boldsymbol{x}_i) - \frac{1}{B} \sum_{j=1}^{B} \log q_{\phi} (\boldsymbol{y}_j \mid \boldsymbol{x}_i) \right].
 \label{eq:hat-vclub}
\end{align}
with samples $\{ ( \boldsymbol{x}_{i}, \boldsymbol{y}_{i} ) \}_{i=1}^{B}$.
After that, the gradient descent is used to optimize $\sigma$.

\section{Methodology}
In this section, we first introduce our proposed multiple-attribute unlearning framework \method{}, which decomposes the task of multiple-attribute unlearning into two steps: \textit{Embedding Calibration} and \textit{Flexible Combination}.
Next, we provide a detailed explanation of these two steps.

\subsection{Overview of \method{}}
To meet the dynamic privacy protection requirements in multiple-attribute unlearning in recommender systems, \method{} performs parallelizable single-attribute unlearning and then combines the unlearned embeddings based on the specific privacy protection requirements.
Figure~\ref{fig:pipline} presents an overview of our proposed \method{}.
After training the recommender system, the user embedding $\boldsymbol{U}_{0}$ of the CF model encode sensitive user information, potentially exposing them to adversaries. 
We denote the sensitive attributes set that needs to be protected under the new privacy protection requirement as $\mathcal{A}_{r} = \left\{\boldsymbol{A}_{1}, \dots, \boldsymbol{A}_{k}\right\}$.

\paragraph{Embedding calibration}
The embedding calibration step modifies the user embedding $\boldsymbol{U}_{0}$ to unlearn a single sensitive attribute $\boldsymbol{A}_{t}$, thereby preventing adversaries from inferring sensitive user information from the embedding while preserving recommendation performance.
After embedding calibration, we obtain $k$ distinct embeddings $\boldsymbol{U}_{1}^{*}, \dots, \boldsymbol{U}_{k}^{*}$, each unlearning the corresponding sensitive attribute $\boldsymbol{A}_{1}, \dots, \boldsymbol{A}_{k}$, respectively.
Although these embeddings protect the unlearned attributes, they may still leak other sensitive user attributes.

\paragraph{Flexible combination}
In this step, embeddings $\boldsymbol{U}_{i}^{*}, i=1, \dots, k$ are combined to form $\boldsymbol{U}^{*} = \alpha_{1} \cdot \boldsymbol{U}_{1}^{*} + \dots + \alpha_{k} \cdot \boldsymbol{U}_{k}^{*}$.
The combination step optimizes only the combination weights $\boldsymbol{\alpha} = [\alpha_{1}, \dots, \alpha_{k}]$, ensuring both flexibility and efficiency.
After the flexible combination, the embedding $\boldsymbol{U}^{*}$ protects all the privacy information that requires protection.
The combined embedding $\boldsymbol{U}^{*}$ then replaces the original user embedding $\boldsymbol{U}_{0}$.

\paragraph{\method{} can meet dynamic requirements}
When a new privacy protection requirement arises:
i) If the new requirement includes new attributes, the embedding calibration step in \method{} can be performed in parallel.
ii) If no new attributes exist, the embedding calibration does not need to be performed again, as embeddings that have already unlearned a specific attribute can be leveraged.
iii) \method{} can swiftly construct a new embedding by combining embeddings that have unlearned a specific attribute, thus meeting the new privacy protection requirement.

\paragraph{\method{} can unlearn multiple attributes simultaneously}
\method{} provides a theoretical guarantee for simultaneously protecting multiple sensitive attributes.
In embedding calibration, we define our unlearning objective as an MI minimization optimization problem with a parameter space constraint.
We minimize MI to prevent adversaries from inferring sensitive user information, while the parameter space constraint preserves recommendation performance.
In flexible combination, we optimize the combination weights by minimizing the MI between the combined embedding and sensitive attributes.
There are several other methods to prevent adversaries from inferring sensitive user information.
Two of the most widely used approaches are distribution alignment (employed in D2DFR) and adversarial training (used in AdvX).
However, these two objectives are not suitable for the two-step approach of \method{}.
The distribution alignment method requires computing the centers of distributions for each attribute.
However, these distributions may differ significantly from one another, thereby combining these embeddings could considerably degrade the recommendation performance of the model.
Since the adversarial training method adversaries different objectives in the first step and is uninterpretable, we cannot guarantee that the combined embedding will effectively protect all sensitive attributes simultaneously.
In contrast, the MI minimization objective ensures that the two-step approach's result does not deviate significantly from the optimal solution.
\begin{definition}
    Let $\boldsymbol{U}_{0}, \boldsymbol{U}^{1}_{i}, \boldsymbol{U}^{2}_{i} \in \mathbb{R}^{N \times d}$ denotes user embeddings,
    \begin{equation*}
        P_1 = \min_{\boldsymbol{\alpha}^{1} \in \Delta^{k-1}} \sum_{t=1}^{k} I \left( \sum_{i=1}^{k} \alpha^{1}_{i} \cdot \boldsymbol{U}^{1}_{i}; \boldsymbol{A}_{t} \right),
    \end{equation*}
    \begin{equation*}
        P_2 = \min_{\boldsymbol{\alpha}^{2} \in \Delta^{k-1}, \boldsymbol{U}_{i}^{2} \in \mathcal{B}_{\epsilon}(\boldsymbol{U}_{0})} \sum_{t=1}^{k} I \left( \sum_{i=1}^{k} \alpha^{2}_{i} \cdot \boldsymbol{U}^{2}_{i}; \boldsymbol{A}_{t} \right),
    \end{equation*}
    where $\Delta^{k-1}$ represents the $(k - 1)$-dimensional standard simplex, $\mathcal{B}_{\epsilon}(\boldsymbol{U}_{0})$ represents the Euclidean ball of radius $\epsilon$ centered at $\boldsymbol{U}_{0}$.
\end{definition}
\begin{theorem}
\label{thm:bound}
    Assume that $\boldsymbol{U}_{i}^{1} = \arg\min_{\boldsymbol{U}_{i}\in \mathcal{B}_{\epsilon}(\boldsymbol{U}_{0})} I\left( \boldsymbol{U}_{i}, \boldsymbol{A}_i\right)$ are constant matrices, 
    and $\Vert \boldsymbol{U}_{0} \Vert_2 \leq C$ for some constant $C > 0$.
    Then, we have the bound $\lvert P_1 - P_2 \rvert \leq 2kL(C + 2\epsilon)$, where $L$ is the Lipschitz constant for MI.
\end{theorem}
\begin{proof}
    The proof can be found in Appendix B.
\end{proof}
%
%
Theorem~\ref{thm:bound} shows that the gap between $P_1$, the result of \method{}, and $P_2$, the result of an end-to-end version of \method{} that unlearns multiple attributes simultaneously, is bounded by $2kL(C + 2\epsilon)$.
This provides a theoretical guarantee that a linear combination of user embeddings with one specific attribute information removed can lead to a user embedding in which all sensitive attribute information is unlearned, demonstrating that \method{} can protect multiple sensitive attributes simultaneously.

\subsection{Embedding Calibration}
In the embedding calibration step, we focus on two objectives in attribute unlearning: to protect a single sensitive attribute while preserving the recommendation performance.

\paragraph{Protecting a single sensitive attribute}
To prevent the sensitive attribute $\boldsymbol{A}_{t}$ from being successfully classified by the attack model, we minimize the MI between the user embedding $\boldsymbol{U}_{0}$ and $\boldsymbol{A}_{t}$.
This can be formalized as follows:
\begin{equation}
    \boldsymbol{U}_{t}^{*} = \argmin_{\boldsymbol{U}_{t}} {I}(\boldsymbol{U}_{t} ; \boldsymbol{A}_{t}).
    \label{eq:opt-ec-ori}
\end{equation}
With a suitable variational distribution, vCLUB provides an upper bound for MI.
Thus, by minimizing the vCLUB, we can effectively minimize the MI:
\begin{equation}
    \boldsymbol{U}_{t}^{*} = \argmin_{\boldsymbol{U}_{t}} {I}_{\text{vCLUB}}(\boldsymbol{U}_{t} ; \boldsymbol{A}_{t}).
    \label{eq:opt-ec-vclub}
\end{equation}
Specifically, we use a neural network parameterized by $\phi$ to model the variational distribution $q_{\phi} (\boldsymbol{A}_{t} \mid \boldsymbol{u}_{t})$.
With vCLUB, we minimize ${I}(\boldsymbol{U}_{t} ; \boldsymbol{A}_{t})$ by minimizing the following objectives through alternating optimization of $\phi$ and $\boldsymbol{U}_{t}$, as detailed in Section~\ref{sec:preliminaries}:
%

\begin{equation}
    \begin{aligned}
        \phi = \argmax_{\phi} \mathbb{E}_{p(\boldsymbol{U}_{t}, \boldsymbol{A}_{t})} \mathcal{L}(\phi), \\
        \boldsymbol{U}_{t}^{*} = \argmin_{\boldsymbol{U}_{t}} \mathbb{E}_{p(\boldsymbol{U}_{t}, \boldsymbol{A}_{t})} \hat{I}_{\text{vCLUB}}.
    \end{aligned}
    \label{eq:opt-ec-woc}
\end{equation}

\paragraph{Perserve recommendation performance}
To preserve the recommendation performance, we apply a parameter space constraint $\boldsymbol{U}_{t} \in \mathcal{B}_{\epsilon}(\boldsymbol{U}_{0})$ to ensure that, after calibration, the embeddings do not deviate significantly from the original ones,
where $\epsilon$ is a hyperparameter that controls the maximum deviation between the calibrated embedding $\boldsymbol{U}_{t}$ and the original embedding $\boldsymbol{U}_{0}$.
Combining the optimization problem described in Eq.~\eqref{eq:opt-ec-woc} with the parameter space constraint, we obtain a constraint optimization problem.
Since the Euclidean projection operator $\text{proj}(\cdot)$ for the constraint has a closed-form solution, we add a projection operation after the alternative optimization algorithm to solve this constrained optimization problem.
Specifically, after updating the embeddings using gradient descent, we apply a projection operation:
\begin{equation}
    \boldsymbol{U}_{t} = \left\{
    \begin{aligned}
        \boldsymbol{U}_{t}, & \quad \text{if}~ \Vert \boldsymbol{U}_{t} - \boldsymbol{U}_{0} \Vert_2 \le \epsilon, \\
        \text{proj}(\boldsymbol{U}_{t}) = \: & \boldsymbol{U}_{0} + \frac{\epsilon}{\Vert \boldsymbol{U}_{t} - \boldsymbol{U}_{0} \Vert_2} (\boldsymbol{U}_{t} - \boldsymbol{U}_{0}), \; \text{otherwise}.
    \end{aligned}
    \right.
    \label{eq:projection}
\end{equation}
%

\subsection{Flexible Combination}
In the flexible combination step, we combine the embeddings to obtain the combined embedding $\boldsymbol{U}^{*} = \boldsymbol{U}(\boldsymbol{\alpha}) = \sum_{i=1}^{k} \alpha_{i} \boldsymbol{U}_{i}^{*}$, where $\boldsymbol{\alpha} = [\alpha_{1}, \dots, \alpha_{k}] \in \mathbb{R}^{k}$.
To ensure that the combined embedding protects all sensitive attributes, we minimize MI between the combined embedding and all sensitive information:
\begin{equation}
\begin{gathered}
    \min_{\boldsymbol{\alpha}} \sum_{i=1}^{k} I \left(\boldsymbol{U}(\boldsymbol{\alpha}) ; \boldsymbol{A}_{i}\right), \\
    \text{s.t.} \quad \alpha_{i} > 0,\, i = 1, ..., k, \quad \Vert \boldsymbol{\alpha} \Vert_{1} = 1.
\end{gathered}
\end{equation}
The constraint in this optimization problem prevents the trivial solution where $\boldsymbol{\alpha} = \boldsymbol{0}$ and ensures the normalization of the weights.
Similarly, we employ vCLUB and an alternative optimization algorithm to minimize MI.
For each attribute $\boldsymbol{A}_{t}$, a neural network parameterized by $\phi_{t}$ is utilized to model the vatiational distribution $q_{\phi_{t}} (\boldsymbol{A}_{t} \mid \boldsymbol{U}_{t})$.
To meet the constraint, we also use the projected gradient descent, where the projection operator is 
\begin{equation}
    \text{proj}(\boldsymbol{\alpha}) = \text{softmax}(\boldsymbol{\alpha}) = \left[ \frac{\exp (\alpha_1)}{\sum_{j=1}^{k} \exp (\alpha_{j})}, \dots,  \frac{\exp (\alpha_k)}{\sum_{j=1}^{k} \exp (\alpha_{j})}\right].
    \label{eq:projection-alpha}
\end{equation}
By using softmax, we ensure that the projection operation adheres to the constraints while maintaining the stability and efficiency of the optimization process.

We summarize the complete procedure of \method{} in Algorithm~\ref{alg:lego-algorithm}.

\begin{algorithm}[h]
\caption{\method{}}
\label{alg:lego-algorithm}
    \begin{algorithmic}[1]
    \STATE \textbf{Input: } User embedding $\boldsymbol{U}_0$, user sensitive attributes $\mathcal{A}_r$, training epoch $E_{1}$, $E_{2}$, batch size $B$, update step size $\eta$, parameter space constraint threshold $\epsilon$.
    \FOR{$t = 1$ to $k$}
        \STATE \textbf{Initial: } $\boldsymbol{U}_{t}^{0} \gets \boldsymbol{U}_{0}$, randomly initialize the parameters of $\phi$.
        \FOR{$e = 0$ to $E_{1}-1$}
            \STATE Sample $\{ ( \theta_{b_{i}}^{\top}, a_{b_{i}}^{t} ) \}_{i=1}^{B}$ from $p(\boldsymbol{U}_{t}^{0}, \boldsymbol{A}_{t})$.
            \STATE Update $\phi$ by maximizing $\mathcal{L}(\phi)$ as defined in Eq.~\eqref{eq:log-likelihood}.
            \STATE Compute MI estimation $\hat{I}_{\text{vCLUB}}$ as defined in Eq.~\eqref{eq:hat-vclub}.
            \STATE $\boldsymbol{U}_{t}^{e+1} \gets \boldsymbol{U}_{t}^{e} - \eta \cdot \nabla_{\boldsymbol{U_{t}^{e}}} \hat{I}_{\text{vCLUB}}$.
            \STATE Project $\boldsymbol{U}_{t}^{e+1}$ as defined in Eq.~\eqref{eq:projection}.
        \ENDFOR
        \STATE $\boldsymbol{U}_{t}^{*} \gets \boldsymbol{U}_{t}^{E_{1}}$.
    \ENDFOR
    \STATE \textbf{Initial: } $\boldsymbol{\alpha}_{0} \gets [ \frac{1}{k}, \dots, \frac{1}{k}]$, randomly initialize the parameters of $\phi_{1}, \dots, \phi_{k}$.
    \FOR{$e = 0$ to $E_{2} - 1$}
        \STATE Sample $\{ ( \theta_{b_{i}}^{\top}, a_{b_{i}}^{1}, \dots, a_{b_{i}}^{k} ) \}_{i=1}^{B}$ from $p(\boldsymbol{U}(\boldsymbol{\alpha}_{0}), \boldsymbol{A}_{1}, \dots, \boldsymbol{A}_{k})$.
        \STATE Update $\phi_{1}, \dots, \phi_{k}$ by maximizing $\mathcal{L}(\phi)$.
        \STATE Compute MI estimation $\hat{I}_{\text{vCLUB}}$.
        \STATE $\boldsymbol{\alpha}_{e+1} \gets \boldsymbol{\alpha}_{e} - \eta \cdot \nabla_{\boldsymbol{\alpha}_{e}} \hat{I}_{\text{vCLUB}}$.
        \STATE Project $\boldsymbol{\alpha}_{e+1}$ as defined in Eq.~\eqref{eq:projection-alpha}.
    \ENDFOR
    \RETURN new user embedding $\boldsymbol{U}(\boldsymbol{\alpha}_{E_{2}})$.
\end{algorithmic}
\end{algorithm}

\section{Experiments}\label{sec:exp}
To comprehensively evaluate our proposed method, we conduct experiments on three benchmark datasets and three representative recommendation models.
Specifically, we aim to answer the following Research Questions (RQs):
\begin{itemize}[leftmargin=*]\setlength{\itemsep}{-\itemsep}
    \item \textbf{RQ1}: Can our method effectively unlearn multiple attributes simultaneously?
    \item \textbf{RQ2}: Does our method preserve the recommendation performance after unlearning?
    \item \textbf{RQ3}: Can our method meet dynamic privacy protection requirements? In other words, how efficient is our proposed approach?
    \item \textbf{RQ4}: What is the impact of key hyperparameters on both unlearning and recommendation performance in our proposed method?
    \item \textbf{RQ5}: What roles do the embedding calibration step and the flexible combination step play in our proposed \method{}?
\end{itemize}

In the Appendix C, we provide additional experimental results for further analysis.

\subsection{Experimental Settings}


\paragraph{Datasets}
We conduct experiments on three publicly available real-world datasets, each containing user-item interaction data and user attribute information (e.g., age and gender).
\begin{itemize}[leftmargin=*]\setlength{\itemsep}{-\itemsep}
\item {\textbf{MovieLens 100K (ML-100K)}\footnote{https://grouplens.org/datasets/movielens/100k/}}: The MovieLens dataset is widely recognized as one of the most extensively used resources for recommender system research~\cite{harper2016movielens}. It contains user ratings for movies, as well as various user attributes such as gender, age, and occupation. Specifically, ML-100K subset includes 100,000 ratings from 1000 users on 1700 movies.
\item {\textbf{MovieLens 1M (ML-1M)}\footnote{https://grouplens.org/datasets/movielens/1m/}}: A version of MovieLens dataset that has 1 million ratings from 6000 users on 4000 movies.
\item {\textbf{KuaiSAR}\footnote{https://kuaisar.github.io/}}: KusiSAR is a large-scale, real-world dataset collected from Kuaishou, a leading short-video app in China with over 350 million daily active users~\cite{sun2023kuaisar}. For users, this dataset included two encrypted features for each user. In our experiments, we utilize KuaiSAR-small.
\end{itemize}
%
%
We provide details of dataset pre-processing and the statistics of the above datasets after pre-processing in Appendix A.
%


\paragraph{Recommendation Models}
We validate the effectiveness of our proposed method across three representative and widely recognized recommendation models.
\begin{itemize}[leftmargin=*]\setlength{\itemsep}{-\itemsep}
\item {\textbf{NCF}}: Neural Collaborative Filtering (NCF) is a foundational collaborative filtering model that employs neural network architectures~\cite{he2017neural}.
\item {\textbf{LightGCN}}: Light Graph Convolution Network (LightGCN) is a State-Of-The-Art (SOTA) collaborative filtering model that optimizes recommendation performance through a simplified graph convolutional network design~\cite{he2020lightgcn}.
\item {\textbf{MultVAE}}: MultVAE learns to recommend items by decoding the variational encoding of user interaction vectors and has shown superior performance compared to various deep neural network approaches~\cite{liang2018variational}.
\end{itemize}

\begin{table*}[t]
\centering
\caption{Results of recommendation performance(HR@10 and NDCG@10) and unlearning performance (i.e., the performance of attackers: BAcc and F1). Except for Original, the best results are highlighted in \textbf{bold}. We run all models 10 times and report the average results and standard deviation. Results are expressed as percentages (\%).}
\label{tab:main_exp}
\resizebox{\linewidth}{!}{
\begin{tabular}{ccc|cccc|cccc|cccc}
\toprule
\multirow{2}{*}{Dataset} & \multirow{2}{*}{Attributes} & \multirow{2}{*}{Method} & \multicolumn{4}{c}{NCF} & \multicolumn{4}{c}{LightGCN} & \multicolumn{4}{c}{MultVAE} \\ 
\cmidrule{4-7} \cmidrule{8-11} \cmidrule{12-15}
  &  &  & HR@10 $\uparrow$ & NDCG@10 $\uparrow$ & BAcc $\downarrow$ & F1 $\downarrow$ & HR@10 $\uparrow$ & NDCG@10 $\uparrow$ & BAcc $\downarrow$ & F1 $\downarrow$ & HR@10 $\uparrow$ & NDCG@10 $\uparrow$ & BAcc $\downarrow$ & F1 $\downarrow$ \\
\midrule

\multirow{15}{*}{ML-100K} & \multirow{5}{*}{Gender} & Original & 15.67±0.32 & 8.37±0.21 & 65.61±0.74 & 66.62±0.60 & 16.07±0.81 & 8.60±0.17 & 62.96±1.73 & 63.43±1.35 & 16.40±1.01 & 8.63±0.32 & 65.64±1.69 & 65.72±1.11 \\
  &  & DP & 5.20±0.28 & 2.40±0.28 & 53.61±0.70 & 52.50±0.46 & 8.03±1.26 & 4.17±0.45 & 61.62±1.16 & 62.15±1.03 & 14.00±0.85 & 6.80±0.14 & 63.11±2.40 & 63.05±2.68 \\
  &  & D2DFR & 16.07±0.15 & 8.23±0.06 & \textbf{47.42±5.57} & 52.13±1.36 & \textbf{16.23±0.15} & \textbf{8.53±0.06} & 55.71±1.36 & 54.22±1.03 & \textbf{16.27±0.76} & \textbf{8.30±0.26} & 59.63±3.18 & 58.97±0.68 \\
  &  & AdvX & \textbf{16.53±1.02} & \textbf{8.63±0.59} & 55.05±7.81 & 60.19±3.42 & 12.30±0.17 & 6.10±0.17 & 54.80±18.30 & 60.21±5.57 & 10.53±0.40 & 5.37±0.15 & 44.15±2.42 & 48.05±1.20 \\
  &  & \method{} & 15.20±0.20 & 7.87±0.12 & 48.18±7.38 & \textbf{47.27±2.10} & 16.00±0.10 & 8.33±0.06 & \textbf{49.09±2.51} & \textbf{48.84±1.00} & 15.83±0.59 & 7.87±0.15 & \textbf{49.30±3.35} & \textbf{49.98±0.46} \\
\cmidrule{2-15}

  & \multirow{5}{*}{Gender, Age} & Original & 15.67±0.32 & 8.37±0.21 & 62.75±0.46 & 63.26±0.44 & 16.07±0.81 & 8.60±0.17 & 60.30±0.69 & 60.54±0.49 & 16.40±1.01 & 8.63±0.32 & 62.21±1.21 & 62.25 \\
  &  & DP & 5.20±0.28 & 2.40±0.28 & \textbf{44.78±0.57} & \textbf{44.22±0.72} & 8.03±1.26 & 4.17±0.45 & 55.97±0.52 & 56.24±0.47 & 14.00±0.85 & 6.80±0.14 & 58.38±1.06 & 58.36±1.18 \\
  &  & D2DFR & 15.63±0.15 & 8.30±0.10 & 48.57±0.47 & 49.84±0.70 & 16.23±0.12 & \textbf{8.50±0.10} & 52.65±1.20 & 51.69±0.36 & \textbf{16.17±0.06} & \textbf{8.33±0.23} & 54.72±0.17 & 54.59±0.49 \\
  &  & AdvX & \textbf{16.23±0.55} & 8.30±0.17 & 54.19±1.07 & 54.42±1.68 & 12.53±0.46 & 6.33±0.12 & 48.20±2.06 & 49.32±0.29 & 7.67±2.07 & 3.80±1.06 & 50.04±1.58 & 50.41±0.17 \\
  &  & \method{} & 15.70±0.00 & \textbf{8.30±0.00} & 46.74±0.73 & 46.54±1.08 & \textbf{16.50±0.00} & 8.40±0.00 & \textbf{36.11±0.84} & \textbf{35.57±0.63} & 16.00±0.78 & 8.07±0.42 & \textbf{38.97±0.29} & \textbf{39.81±1.04} \\
\cmidrule{2-15}

  & \multirow{5}{*}{\makecell[c]{Gender, Age, \\Occupation}} & Original & 15.67±0.32 & 8.37±0.21 & 42.67±0.11 & 42.97±0.03 & 16.07±0.81 & 8.60±0.17 & 41.78±0.78 & 42.16±0.63 & 16.40±1.01 & 8.63±0.32 & 42.83±0.92 & 42.93±0.58 \\
  &  & DP & 5.20±0.28 & 2.40±0.28 & \textbf{30.76±0.23} & \textbf{30.27±0.40} & 8.03±1.26 & 4.17±0.45 & 38.98±0.37 & 39.51±0.39 & 14.00±0.85 & 6.80±0.14 & 40.73±0.71 & 40.28±0.71 \\
  &  & D2DFR & 15.57±0.12 & 8.23±0.06 & 33.96±0.66 & 33.65±0.40 & 16.37±0.06 & 8.50±0.00 & 35.01±0.55 & 35.30±0.81 & 15.83±0.40 & 8.27±0.38 & 37.69±0.79 & 38.15±0.39 \\
  &  & AdvX & 15.83±0.58 & \textbf{8.40±0.35} & 33.99±3.60 & 35.71±1.34 & 11.97±0.31 & 6.07±0.06 & 34.37±7.63 & 41.67±0.05 & 6.25±2.90 & 3.50±0.99 & \textbf{28.01±0.93} & \textbf{27.83±1.49} \\
  &  & \method{} & \textbf{16.00±0.00} & 8.03±0.06 & 32.30±0.28 & 33.22±0.22 & \textbf{16.60±0.00} & \textbf{8.60±0.00} & \textbf{28.19±0.99} & \textbf{27.59±0.98} & \textbf{16.33±0.50} & \textbf{8.40±0.20} & 30.85±1.67 & 30.88±0.98 \\
\midrule

\multirow{15}{*}{ML-1M} & \multirow{5}{*}{Gender} & Original & 8.20±0.26 & 4.10±0.17 & 76.12±0.20 & 75.60±0.13 & 9.10±0.10 & 4.60±0.10 & 71.28±00.54 & 70.41±0.70 & 9.20±0.26 & 4.40±0.20 & 72.10±1.26 & 71.18±1.08 \\
  &  & DP & 2.20±0.10 & 0.97±0.06 & 55.30±1.01 & 55.32±0.57 & 3.13±0.06 & 1.53±0.06 & 66.26±0.21 & 65.98±0.24 & 7.67±0.26 & 3.70±0.00 & 68.46±0.58 & 68.06±0.58 \\
  &  & D2DFR & 8.50±0.00 & 4.20±0.00 & 51.28±6.52 & 50.64±0.17 & 5.90±0.00 & 3.00±0.00 & \textbf{49.03±5.97} & 55.41±1.64 & 8.03±0.12 & 4.03±0.06 & \textbf{47.86±7.44} & 54.82±2.38 \\
  &  & AdvX & \textbf{8.60±0.26} & \textbf{4.27±0.15} & 55.21±3.10 & 61.00±2.87 & 4.57±0.06 & 2.20±0.00 & 64.17±6.76 & 68.49±2.23 & 4.67±0.15 & 2.37±0.06 & 63.65±1.93 & 64.94±0.45 \\
  &  & \method{} & 8.07±0.06 & 3.97±0.06 & \textbf{47.48±0.63} & \textbf{44.83±0.23} & \textbf{8.70±0.00} & \textbf{4.50±0.00} & 51.20±2.21 & \textbf{49.56±0.18} & \textbf{8.97±0.46} & \textbf{4.40±0.10} & 53.21±0.42 & \textbf{52.17±0.80} \\
\cmidrule{2-15}

  & \multirow{5}{*}{Gender, Age} & Original & 8.20±0.26 & 4.10±0.17 & 71.88±0.13 & 71.62±0.16 & 9.10±0.10 & 4.60±0.10 & 67.45±0.26 & 67.01±0.46 & 9.20±0.26 & 4.40±0.20 & 68.10±0.38 & 67.64±0.23 \\
  &  & DP & 2.20±0.10 & 0.97±0.06 & \textbf{47.76±0.61} & \textbf{47.77±0.48} & 3.13±0.06 & 1.53±0.06 & 60.27±0.49 & 60.13±0.54 & 7.67±0.23 & 3.70±0.00 & 64.50±0.68 & 64.30±0.45 \\
  &  & D2DFR & \textbf{8.40±0.00} & \textbf{4.17±0.06} & 61.09±0.31 & 60.98±0.16 & 5.90±0.00 & 3.00±0.00 & \textbf{44.66±6.25} & 54.31±1.42 & 7.93±0.15 & 3.97±0.06 & 50.22±2.41 & 53.90±1.15 \\
  &  & AdvX & 8.33±0.31 & 4.10±0.10 & 46.65±6.92 & 49.62±1.44 & 4.63±0.06 & 2.30±0.00 & 61.04±0.74 & 61.26±0.76 & 4.43±0.23 & 2.13±0.06 & \textbf{40.98±9.58} & \textbf{41.74±0.17} \\
  &  & \method{} & 8.10±0.00 & 4.00±0.00 & 56.34±0.20 & 55.68±0.29 & \textbf{8.60±0.00} & \textbf{4.40±0.00} & 50.80±0.44 & \textbf{50.13±0.19} & \textbf{8.80±0.00} & \textbf{4.27±0.06} & 53.09±0.34 & 52.29±0.42 \\
\cmidrule{2-15}

  & \multirow{5}{*}{\makecell[c]{Gender, Age, \\Occupation}} & Original & 8.20±0.26 & 4.10±0.17 & 51.25±0.45 & 51.07±0.50 & 9.10±0.10 & 4.60±0.10 & 48.32±0.36 & 48.07±0.55 & 9.20±0.26 & 4.40±0.20 & 49.41±0.53 & 49.17±0.47 \\
  &  & DP & 2.20±0.10 & 0.97±0.06 & \textbf{33.46±0.50} & \textbf{33.43±0.38} & 3.13±0.06 & 1.53±0.06 & 43.26±0.49 & 43.11±0.51 & 7.67±0.23 & 3.70±0.00 & 45.67±0.46 & 45.54±0.34 \\
  &  & D2DFR & \textbf{8.40±0.00} & \textbf{4.10±0.00} & 42.20±0.14 & 41.89±0.17 & 6.00±0.00 & 3.00±0.00 & 43.91±0.16 & 43.48±0.25 & 7.97±0.15 & 4.00±0.00 & 45.39±0.54 & 45.56±0.39 \\
  &  & AdvX & 8.30±0.36 & 4.00±0.20 & 44.32±0.60 & 44.33±0.50 & 4.50±0.00 & 2.20±0.00 & 39.66±0.32 & 39.77±0.38 & 4.57±0.21 & 2.30±0.20 & \textbf{26.28±5.59} & \textbf{29.26±0.06} \\
  &  & \method{} & 8.23±0.06 & 4.00±0.00 & 41.55±0.10 & 41.02±0.02 & \textbf{8.80±0.00} & \textbf{4.50±0.00} & \textbf{38.97±0.28} & \textbf{38.32±0.63} & \textbf{9.07±0.42} & \textbf{4.40±0.17} & 39.55±0.45 & 39.12±0.31 \\
\midrule

\multirow{10}{*}{KuaiSAR} & \multirow{5}{*}{Feat1} & Original & 1.87±0.12 & 0.93±0.06 & 14.87±0.66 & 14.88±0.66 & 3.43±0.06 & 1.80±0.00 & 13.30±1.40 & 13.30±1.41 & 3.30±0.00 & 1.67±0.06 & 13.59±1.33 & 13.59±1.32 \\
  &  & DP & 1.33±0.06 & 0.70±0.00 & \textbf{13.45±1.31} & \textbf{13.45±1.32} & 1.27±0.06 & 0.67±0.06 & 12.37±0.85 & 12.37±0.84 & 2.73±0.06 & 1.40±0.00 & 13.52±0.15 & 13.51±0.14 \\
  &  & D2DFR & \textbf{2.00±0.00} & \textbf{1.00±0.00} & 14.08±1.18 & 14.08±1.20 & \textbf{3.53±0.06} & \textbf{1.80±0.00} & 12.98±0.70 & 12.97±0.70 & \textbf{3.30±0.00} & 1.63±0.06 & 13.26±0.65 & 13.26±0.65 \\
  &  & AdvX & 1.53±0.25 & 0.73±0.15 & 14.69±0.84 & 14.75±0.84 & 2.47±0.06 & 1.27±0.06 & 12.87±0.53 & 12.87±0.55 & 1.33±0.25 & 0.70±0.17 & \textbf{12.54±0.04} & \textbf{12.50±0.00} \\
  &  & \method{} & \textbf{2.00±0.00} & \textbf{1.00±0.00} & 14.41±0.04 & 14.43±0.04 & 3.50±0.00 & \textbf{1.80±0.00} & \textbf{11.57±0.16} & \textbf{11.56±0.17} & 3.30±0.10 & \textbf{1.67±0.06} & 12.59±0.72 & 12.58±0.73 \\
\cmidrule{2-15}

  & \multirow{5}{*}{Feat1, Feat2} & Original & 1.87±0.12 & 0.93±0.06 & 24.01±0.66 & 24.01±0.65 & 3.43±0.06 & 1.80±0.00 & 23.67±0.37 & 23.68±0.36 & 3.30±0.00 & 1.67±0.06 & 24.01±1.09 & 24.01±1.08 \\
  &  & DP & 1.33±0.06 & 0.70±0.00 & \textbf{21.77±1.40} & \textbf{21.78±1.40} & 1.27±0.06 & 0.67±0.06 & 21.49±0.82 & 21.48±0.82 & 2.73±0.06 & 1.40±0.00 & 24.03±0.30 & 24.02±0.29 \\
  &  & D2DFR & \textbf{2.00±0.00} & \textbf{1.00±0.00} & 22.16±0.98 & 22.16±0.98 & \textbf{3.50±0.00} & \textbf{1.80±0.00} & 19.33±0.76 & 19.33±0.77 & \textbf{3.30±0.10} & \textbf{1.67±0.06} & 23.78±0.79 & 23.76±0.79 \\
  &  & AdvX & 1.83±0.15 & 0.87±0.06 & 24.83±0.73 & 24.88±0.71 & 1.93±0.06 & 1.03±0.06 & 23.21±0.57 & 23.25±0.50 & 1.20±0.70 & 0.57±0.32 & 23.78±0.50 & 23.75±0.47 \\
  &  & \method{} & \textbf{2.00±0.00} & \textbf{1.00±0.00} & 23.99±0.44 & 23.67±0.43 & \textbf{3.50±0.00} & \textbf{1.80±0.00} & \textbf{18.47±0.12} & \textbf{18.73±0.11} & 2.23±0.06 & 1.10±0.00 & \textbf{22.97±0.78} & \textbf{22.97±0.78} \\

\bottomrule
\end{tabular}
}
\end{table*}

\paragraph{Unlearning Methods}
We compare our proposed method, \method{}, with the original model and three attribute unlearning methods.
\begin{itemize}[leftmargin=*]\setlength{\itemsep}{-\itemsep}
\item {\textbf{Original}}: This is the original model without attribute unlearning.
\item {\textbf{DP}~\cite{zhu2016differential}}: This method protects user attributes by introducing noise perturbation to the user embedding during the model prediction process.
\item {\textbf{D2DFR}~\cite{chen2024posttraining}}: This method represents the latest SOTA single-attribute unlearning method, which is achieved through distribution alignment. 
To extend this method to multi-attribute unlearning, we adopt a sequential forgetting approach, where after forgetting one attribute, the method continues to forget the next attribute until all attributes have been forgotten.
\item {\textbf{AdvX}~\cite{escobedo2024simultaneous}}: This is the only multiple-attribute unlearning method, which employs adversarial training to achieve attribute unlearning. While the original method is specifically designed for MultVAE, we extend it to other recommendation models.
\end{itemize}

\begin{figure*}[h]
    \centering
    \includegraphics[width=0.8\linewidth]{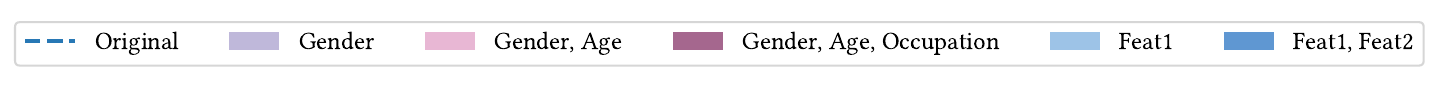} \\
    \vspace{-1.2em}
    \subfigure[\textbf{ML-100K dataset}]{\includegraphics[width=0.33\linewidth]{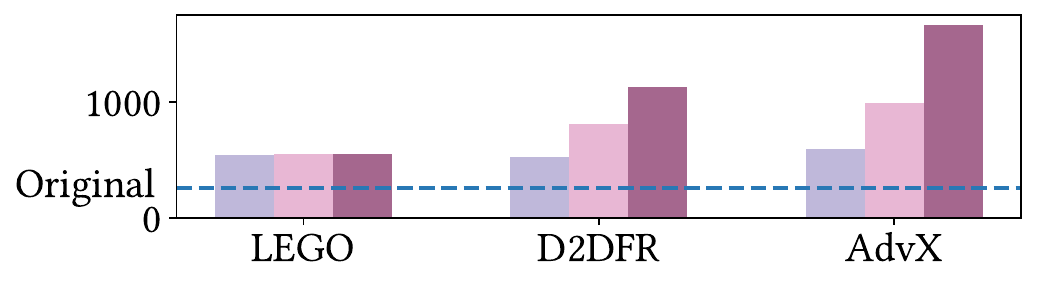}} \label{fig:efficiency_ml_100k}
    \subfigure[\textbf{ML-1M dataset}]{\includegraphics[width=0.33\linewidth]{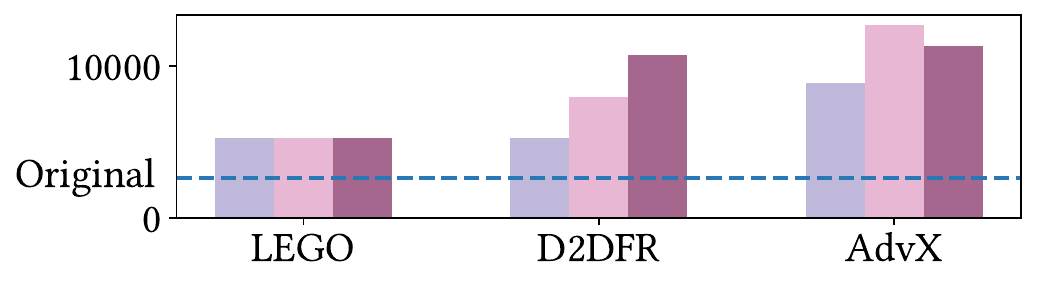}} \label{fig:efficiency_ml_1m}
    \subfigure[\textbf{KuaiSAR dataset}]{\includegraphics[width=0.33\linewidth]{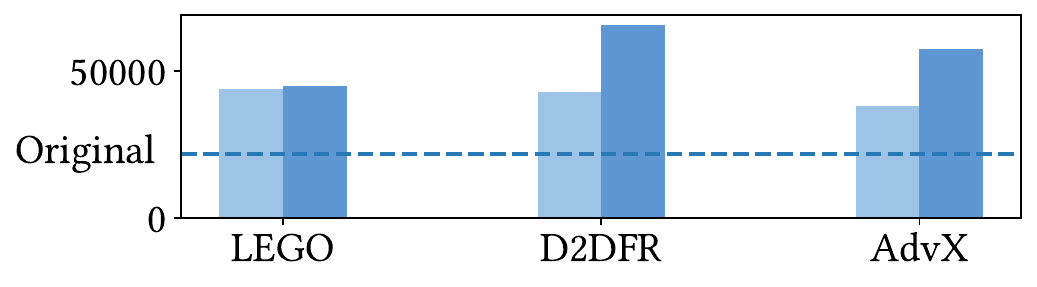}} \label{fig:efficiency_kuaisar}
    \caption{Results of efficiency in adapting to dynamic requirements. We present the running time of compared methods on NCF model across three datasets. We run all models 10 times and report the average results in seconds (s). The dashed line represents the training time of the original recommendation model.}
    \label{fig:efficiency}
\end{figure*}

We provide details of evaluation metrics, parameter settings, and hardware information in Appendix A.




\subsection{Results and Discussions}
\subsubsection{Unlearning Performance (RQ1)}
The primary goal of attribute unlearning is to remove sensitive information from the recommendation model, preventing adversaries from inferring sensitive user attributes.
To comprehensively evaluate the unlearning performance of \method{}, we report two metrics, F1 score and BAcc, in Table~\ref{tab:main_exp}.
%
%
DP, D2DFR, AdvX, and \method{} reduce the BAcc by an average of 12.77\%, 20.84\%, 18.37\%, and 24.31\%, respectively, compared to the original model.
These results demonstrate that \method{} effectively removes sensitive information from the recommendation model.
Specifically, D2DFR reduces the BAcc on one, two, and three attributes by an average of 25.08\%, 22.58\%, and 13.79\%, respectively, indicating that D2DFR is less effective at removing multiple attributes simultaneously.
AdvX reduces the BAcc on MultVAE by an average of 24.75\%, and on NCF and LightGCN by 19.61\% and 13.04\%, respectively, highlighting that AdvX lacks of generalizability across different recommendation models.

\subsubsection{Recommendation Performance (RQ2)}
%
While unlearning sensitive user attributes, the impact on recommendation performance should be minimized to ensure the utility of the recommender system.
We use HR and NDCG to evaluate recommendation performance after unlearning, truncating the rank list at 10 for both metrics.
As shown in Table~\ref{tab:main_exp}, unlearning methods do affect recommendation performance to varying degrees.
%
%
DP, D2DFR, AdvX, and \method{} reduce NDCG@10 by 48.17\%, 5.64\%, 30.30\%, and 3.43\%, respectively, on average.
The results demonstrate that \method{} effectively preserves recommendation performance.
Specifically, D2DFR decreases NDCG@10 on one, two, and three attributes by an average of 5.52\%, 5.48\%, and 6.72\%, respectively.
In contrast, \method{} reduces NDCG@10 on one, two, and three attributes by an average of 3.92\%, 4.08\%, and 1.99\%, respectively.
This indicates that while the sequential unlearning methods degrade model recommendation performance, \method{} does not have the same effect.

\subsubsection{Efficiency in Adapting to Dynamic Requirements (RQ3)}
We evaluate the efficiency of these unlearning methods in adapting to dynamic privacy protection requirements based on their running time.
Since the recommendation model does not affect the overall trend, We conduct experiments on all three datasets using the NCF model, with the total running time reported in seconds.
For better comparison, we indicate the original recommendation model training time with a dashed line.
DP only adds noise to the inference process, so its running time is the same as the original training time.
As shown in Figure~\ref{fig:efficiency}, we observe that compared to AdvX, our proposed \method{} significantly reduces running time across all three datasets.
This is because AdvX employs a time-consuming adversarial training approach during its training process to achieve attribute unlearning.
%
%
%
D2DFR's running time is directly proportional to the number of attributes.
Specifically, our proposed \method{} achieves nearly the same efficiency in multiple attributes unlearning as in single-attribute unlearning.
These results demonstrate that \method{} can effectively meet dynamic privacy protection requirements.


\subsubsection{Parameter Sensitivity (RQ4)}
We investigate the hyperparameter $\epsilon$, which controls the maximum deviation of the unlearned embedding from the original embedding.
This parameter trades off unlearning effectiveness and recommendation performance.
Since the total norm is related to the number of users $N$, we control $\epsilon / N$ to be 0, 0.1, 0.2, 0.3, 0.4, and $\infty$ (without any constraint).
In the experiment, we report the results of unlearning all user attributes recorded in the dataset, while fixing the number of training iterations at 2000 to ensure convergence.
As shown in Figure~4, particularly in Figure~4(b), as $\epsilon / N$ increases, both NDCG@10 and BAcc decrease.
This occurs because a looser constraint allows for more extensive calibration of the embedding to improve attribute unlearning effectiveness, but it degrades recommendation performance.
In Figure~4(b), LightGCN's NDCG@10 increases as the $\epsilon / N$ increases.
This is because our method may unintentionally reduce negative biases, potentially leading to unexpected improvements in recommendation performance. This phenomenon has been consistently observed in prior work~\cite{li2023making, chen2024posttraining}.
As shown in Figure~4(a), in the ML-100K dataset, recommendation performance remains robust to changes in $\epsilon / N$, as it is a relatively small dataset.
Due to space constraints, the full set of hyperparameter sensitivity results are provided in the Appendix C.

\subsubsection{Ablation Study (RQ5)}
Table~\ref{tab:ablation} presents the results of an ablation study conducted using NCF on the ML-1M dataset. 
We sequentially remove the embedding calibration step and the flexible combination step to assess their impact on the unlearning (F1 and BAcc) and recommendation (HR and NDCG) performance. 
Initially, when we replace the embedding calibration step with D2DFR to unlearn a single attribute (D2DFR-FC), we observe a significant increase in F1 score and BAcc and a significant decrease in HR and NDCG. 
This indicates that without the embedding calibration step using MI minimization, the flexible combination step cannot guarantee unlearning effectiveness. 
Subsequently, we remove the flexible combination step and combine the embeddings by averaging them (EC-AC).
Although the model performed well in recommendation performance, there is a noticeable decline in unlearning performance. 
This suggests that combining the embeddings by averaging them may result in suboptimal weights, thereby reducing unlearning performance. 
The results of the ablation study clearly demonstrate the critical roles of both steps in our proposed \method{}. 
%

\begin{table}
\centering
\caption{Results of ablation studies on two steps (Step 1: D2DFR-FC, Step 2: EC-AC).}
\label{tab:ablation}
\resizebox{\linewidth}{!}{
\begin{tabular}{ccccc}
\toprule
 & HR@10 (\%) $\uparrow$ & NDCG@10 (\%) $\uparrow$ & BAcc (\%) $\downarrow$ & F1 (\%) $\downarrow$ \\ 
\midrule
 D2DFR-FC & 7.56±0.04 & 3.24±0.03 & 47.65±1.00 & 47.39±0.80 \\
 EC-AC & 8.19±0.06 & 3.96±0.01 & 45.57±0.06 & 44.35±0.05 \\
 \method{} & 8.23±0.06 & 4.00±0.00 & 41.55±0.10 & 41.02±0.02 \\
 \bottomrule
\end{tabular}
}
\end{table}

\section{Conclusion}\label{sec:conclusion}
In this paper, we investigate multiple-attribute unlearning in recommender systems, aiming to simultaneously remove multiple sensitive attributes while efficiently adapting to dynamic privacy protection requirements. 
To the best of our knowledge, we are the first to identify the dynamic privacy protection requirements that often involve multiple sensitive attributes and evolve over time and across regions.
Existing single-attribute unlearning methods fail to meet these requirements due to two key challenges: i) \textbf{CH1}: the inability to handle multiple unlearning requests simultaneously, and ii) \textbf{CH2}: the lack of adaptability to dynamic unlearning needs.
To address these challenges, we propose \method{}, which decomposes multiple-attribute unlearning into two steps: \textit{Embedding Calibration} and \textit{Flexible Combination}.
We conduct extensive experiments on three real-world datasets and three representative recommendation models to evaluate the effectiveness and efficiency of our proposed method.
The results demonstrate that \method{} achieves performance comparable to baseline methods in single-attribute unlearning and outperforms them in multiple-attribute unlearning while preserving recommendation performance.
Furthermore, our method proves to be highly efficient in adapting to dynamic privacy protection requirements.
Note that all existing work focuses on discrete attributes or uses binning to transform continuous attributes into discrete ones, yet continuous attributes are prevalent in the real world.

\begin{acks}
This work was supported in part by the Zhejiang Provincial Natural Science Foundation of China (No.~LZYQ25F020002), Hangzhou Key Scientific Research Plan (No.~2024SZD1A28), and the National Natural Science Foundation of China (No.~62402148).
\end{acks}

\bibliographystyle{ACM-Reference-Format}
\balance
\bibliography{reference}

\clearpage

\appendix

\section{Experimental Details}\label{sec:exp_detail}
\paragraph{Dataset pre-processing}

\begin{table}[h]
\centering
\caption{Statistics of datasets after pre-processing.}
\label{tab:dataset}
\resizebox{\linewidth}{!}{
\begin{tabular}{lcccccc}
\toprule
Dataset & Attribute & Category \# & User \#  & Item \#   & Rating \# & Sparsity \\ \midrule
\multirow{3}{*}{ML-100K} & Gender & 2 & \multirow{3}{*}{943} & \multirow{3}{*}{1,682} & \multirow{3}{*}{100,000} & \multirow{3}{*}{93.695\%} \\
 & Age & 3 &  &  &  &  \\
 & Occupation & 21 &  &  &  &  \\ \midrule
\multirow{3}{*}{ML-1M} & Gender & 2 & \multirow{3}{*}{6,040} & \multirow{3}{*}{3,706} & \multirow{3}{*}{1,000,209} & \multirow{3}{*}{95.531\%} \\
 & Age & 3 &  &  &  &  \\
 & Occupation & 21 &  &  &  &  \\ \midrule
\multirow{2}{*}{KuaiSAR} 
 & Feat1 & 8 & \multirow{2}{*}{25,473} & \multirow{2}{*}{284,996}   & \multirow{2}{*}{4,619,183}  & \multirow{2}{*}{99.936\%} \\
 & Feat2 & 3 &  &  &  \\
\bottomrule
\end{tabular}
}
\end{table}

For ML-100K and ML-1M, we retrain only users who have interacted with at least five items.
For KuaiSAR, we retain only users who have interacted with at least five items and items that have received at least five user interactions.
To assess recommendation performance, the most recent interaction items for each user (sorted by interaction timestamp) are retained for testing.
For ML-100K and ML-1M, the available gender attribute is restricted to male and female categories.
The age attribute is divided into three groups: under 28 years old, between 28 and 40, and over 40 for ML-100K, and under 25, between 25 and 35, and over 35 for ML-1M.
For KuaiSAR, we use anonymized one-hot encoded categories of users as the target attributes.
We summarize the statistics of the above datasets after pre-processing in Table~\ref{tab:dataset}.

It is worth noting that we do not use the LFM-2B dataset, which has been widely used in previous work, because it is not available for download due to license issues.

\paragraph{Evaluation Metrics}
We specify the evaluation metrics of unlearning effectiveness and recommendation performance as follows.

As mentioned in Section~\ref{sec:preliminaries}, the attack process is considered a classification task, where the attack model takes user embeddings as input and the attributes as labels.
Following \cite{li2023making, chen2024posttraining, tao2023dudb}, we build a Multilayer Perceptron (MLP)~\cite{gardner1998artificial} as the adversarial classifier, since MLP demonstrates the best performance as the attacker, as shown in~\cite{li2023making}.
The dimension of MLP's hidden layer is set to 100, with a softmax layer as the output layer. 
We set the L2 regularization weight to 1.0, the initial learning rate to 1e-2, and the maximum number of iterations to 500, leaving the other hyperparameters at their default values in scikit-learn 1.4.2.
We train the MLP using 80\% of the users and test it on the remaining 20\%.
To evaluate the effectiveness of attribute unlearning, we use two widely adopted classification metrics: the micro-averaged F1 score (F1) and Balanced Accuracy (BAcc). 
Lower values of F1 and BAcc indicate greater effectiveness of attribute unlearning.
We report the results of the attack using five-fold cross-validation.
For the results of the multiple-attribute attack, we report the average F1 and BAcc across all attributes.

To assess recommendation performance, we use leave-one-out testing~\cite{he2016fast, xu2025robust, zhang2023multi}.
We use Hit Ratio at rank K (HR@K) and Normalized Discounted Cumulative Gain at rank K (NDCG@K) as metrics to evaluate recommendation performance.
HR@K measures whether the test item is in the top-K list, while NDCG@K is a position-aware ranking metric that gives higher scores to hits that occur at higher ranks.
In our experiment, the entire negative item set is used to compute HR@K and NDCG@K.
Note that we compare the recommendation performance of several methods under the condition of achieving optimal unlearning effectiveness.

\paragraph{Training parameters}
For model-specific parameters in the recommendation models, we follow the settings provided in the respective original papers.
Specifically, we use the Adam optimizer with a learning rate of 1e-3 and set the embedding dimension to 32 for NCF and LightGCN, and 200 for MultVAE.
The number of epochs is set to 20, 100, and 20 for NCF, LightGCN, MultVAE, respectively.

\paragraph{Details of Applying Single-Attribute Methods for Multi-Attribute Unlearning}
We apply single-attribute unlearning methods sequentially by removing one attribute at a time. For example, to unlearn Gender, Age, and Occupation on MovieLens using D2DFR, we first apply D2DFR to remove Gender, obtaining an intermediate model $\mathcal{M}_1$; then apply D2DFR on $\mathcal{M}_1$ to unlearn Age, resulting in $\mathcal{M}_2$; finally, apply D2DFR again on $\mathcal{M}_2$ to unlearn Occupation, yielding the final unlearned model.

\paragraph{Hyperparameters}
To obtain the optimal performance for all methods, we use grid search to tune the hyperparameters.
In D2DFR, we set the trade-off coefficient 1e-6.
In AdvX, we set the gradient scaling coefficient to be 600.
In our proposed \method{}, we set $\epsilon / N$ to 0.5. 
We use the Adam optimizer with a learning rate of 1e-3 to optimize the embeddings during the embedding calibration step.
We construct a two-layer MLP as the variational distribution, with a hidden layer dimension of 100 and a softmax output layer. The learning rate of the MLP is set to 1e-4, and the training is run for 2000 iterations to ensure convergence.

\paragraph{Hardware information}
All models and algorithms are implemented using Python 3.9 and Pytorch 2.3.0.
The experiments are conducted on a server running Ubuntu 22.04, equipped with 256GB of RAM and an NVIDIA GeForce RTX 4090 GPU.

\begin{figure*}[h]
    \centering
    \includegraphics[width=0.8\linewidth]{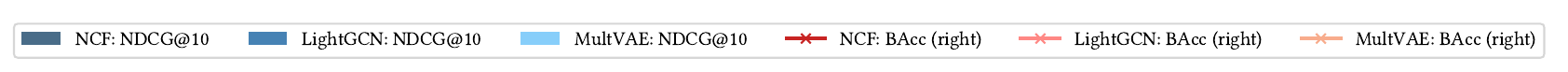} \\
    \vspace{-1.2em}
    \subfigure[\textbf{ML-100K dataset}]{\includegraphics[width=0.32\linewidth]{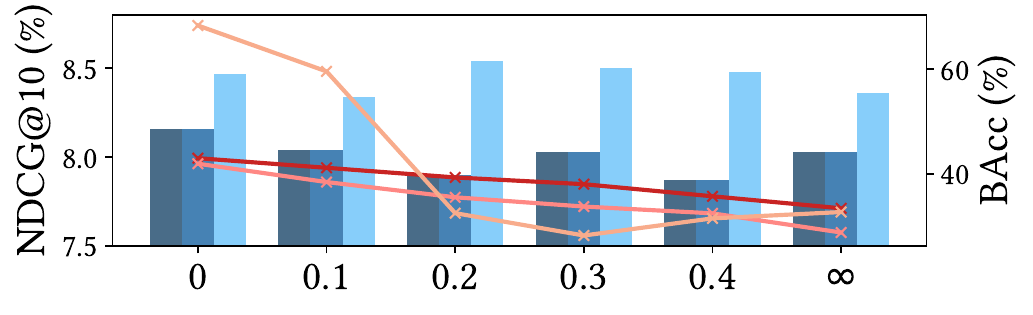} \label{fig:hyper_ml_100k}}
    \subfigure[\textbf{ML-1M dataset}]{\includegraphics[width=0.32\linewidth]{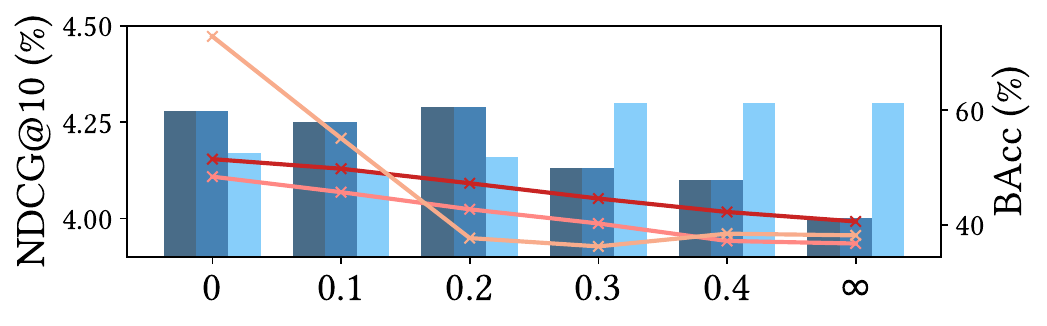} \label{fig:hyper_ml_1m}}
    \subfigure[\textbf{KuaiSAR dataset}]{\includegraphics[width=0.32\linewidth]{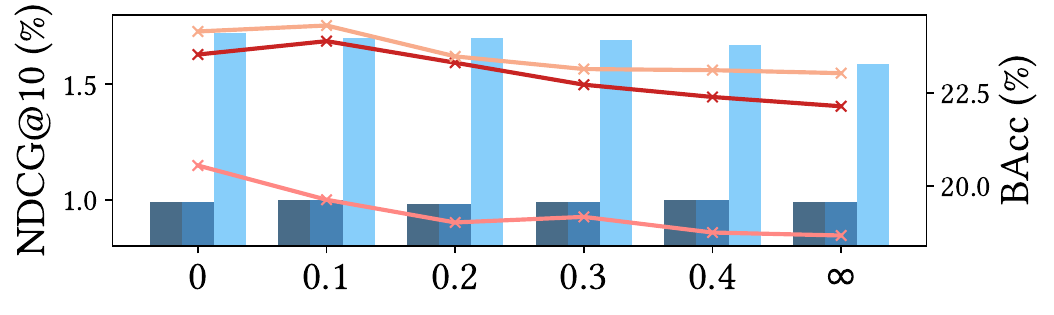} \label{fig:hyper_kuaisar}}
    \caption{Effect of hyperparameter $\epsilon$. We conduct experiments on all three models across three datasets. We use BAcc and NDCG@10 to represent the performance of unlearning and recommendation respectively. We report the results of unlearning all user attributes recorded in the dataset.}
    \label{fig:hyperparam}
\end{figure*}

\section{Proof of Theorem 1}\label{sec:appendix}
\begin{proof}
    For clarity of notation, let us define the combined embedding and the MI between the combined embedding and the sensitive attribute as follows:
    \begin{equation*}
    \begin{gathered}
        \boldsymbol{U}^{q_{1}} \left( \boldsymbol{\alpha}^{q_{2}} \right) = \sum_{i=1}^{k} \left( \alpha_{i}^{q_{2}} \cdot \boldsymbol{U}_{i}^{q_{1}} \right), \\
        I^{(q_1, q_2)}_{t} = I \left( \boldsymbol{U}^{q_1} \left( \boldsymbol{\alpha}^{q_2} \right); \boldsymbol{A}_{t} \right).
    \end{gathered}
    \end{equation*}
    Using the triangle inequality, we can split the $\vert P_1 - P_2 \vert$ into two terms:
    \begin{equation*}
    \begin{aligned}
        \lvert P_1 - P_2 \rvert & = \left\lvert \sum_{t=1}^{k} \left( I^{(1, 1)}_{t} - I^{(2, 2)}_{t} \right) \right\rvert \\
        & \leq \left\lvert \sum_{t=1}^{k} \left( I^{(1, 1)}_{t} - I^{(2, 1)}_{t} \right) \right\rvert + \left\lvert \sum_{t=1}^{k} \left( I^{(2, 1)}_{t} - I^{(2, 2)}_{t} \right) \right\rvert.
    \end{aligned}
    \end{equation*}
    Applying the Lipschitz continuity of MI with respect to its first argument gives us the bounds of these two terms:
    \begin{equation*}
    \begin{aligned}
        \left\lvert \sum_{t=1}^{k} \left( I^{(1, 1)}_{t} - I^{(2, 1)}_{t} \right) \right\rvert 
        %
        %
        & \leq \sum_{t=1}^{k} L \left( \sum_{i=1}^{k} \left\lvert \alpha^{1}_{i} \right\rvert \cdot \left\lVert \boldsymbol{U}^{1}_{i} - \boldsymbol{U}^{2}_{i} \right\rVert_2 \right) \\
        & \leq \sum_{t=1}^{k} 2L\epsilon \left( \sum_{i=1}^{k} \left\lvert \alpha^{1}_{i} \right\rvert \right)
        = 2kL\epsilon, \\
        \left\lvert \sum_{t=1}^{k} \left( I^{(2, 1)}_{t} - I^{(2, 2)}_{t} \right) \right\rvert
        & \leq \sum_{t=1}^{k} L \left\lVert \sum_{i=1}^{k} \left( \alpha^{1}_{i} - \alpha^{2}_{i} \right)\cdot \boldsymbol{U}^{2}_{i} \right\rVert_2 \\
        %
        %
        & \leq 2kL(C + \epsilon),
    \end{aligned}
    \end{equation*}
    where $L > 0$ is the Lipschitz constant.
    Combining these two inequality, the total gap is bounded by:
    \begin{equation*}
        \vert P_1 - P_2 \vert \leq 2kL\epsilon + 2kL(C + \epsilon) = 2kL(C + 2\epsilon).
    \end{equation*}
\end{proof}

\section{Additional Experimental Results}\label{sec:additional_exp}
\subsection{Unlearning Correlated Attributes}
\begin{table}[h]
\centering
\caption{Results of evaluating the impact of unlearning one attribute on the inference performance of a correlated attribute. The experiment is conducted on the LightGCN model using the ML-1M dataset.}
\label{tab:correlated}
\resizebox{\linewidth}{!}{
\begin{tabular}{ccccc}
\toprule
Attribute & Gender F1 & Gender BAcc & Occupation F1 & Occupation BAcc \\ \midrule
Original & 70.41 & 71.28 & 10.29 & 11.38 \\
Gender & 49.56 & 51.20 & 9.05 & 9.36 \\
Occupation & 68.34 & 70.72 & 4.65 & 4.68 \\
Gender, Occupation & 55.63 & 55.09 & 5.68 & 5.84 \\
\bottomrule
\end{tabular}
}
\end{table}

Additionally, we conduct an experiment on the LightGCN model using the ML-1M dataset to evaluate how unlearning one attribute affects the unlearning performance of a correlated attribute (e.g., gender and occupation), with results shown in the Table~\ref{tab:correlated}. We observe that unlearning one attribute slightly reduces F1 and BAcc on a correlated attribute, but the values remain higher than those after \method{}. These results indicate that single-attribute unlearning provides some unintended privacy protection on correlated attributes, but \method{} remains necessary for effective multi-attribute unlearning, as it achieves lower AIA accuracy overall.

\subsection{Empirical Validation of the Theoretical Bound}
\begin{table}[h]
\centering
\caption{Results of evaluating the empirical tightness of the theoretical bound in Theorem~\ref{thm:bound} across datasets and models.}
\label{tab:bound}
\resizebox{0.8\linewidth}{!}{
\begin{tabular}{ccccc}
\toprule
Dataset & MI & NCF & LightGCN & MultVAE \\ \midrule
\multirow{2}{*}{ML-100K} & $P_1$ & 0.4850 & 0.7478 & 0.8180 \\
& $P_2$ & 0.4665 & 0.7239 & 0.7883 \\ \midrule
\multirow{2}{*}{ML-1M} & $P_1$ & 0.5040 & 0.7858 & 0.8655 \\
& $P_2$ & 0.4843 & 0.7535 & 0.8502 \\ \midrule
\multirow{2}{*}{KuaiSAR} & $P_1$ & 0.0240 & 0.0043 & 0.0199 \\
& $P_2$ & 0.0209 & 0.0041 & 0.0114 \\
\bottomrule
\end{tabular}
}
\end{table}

Theorem~\ref{thm:bound} provides a theoretical guarantee that a linear combination of user embeddings with one specific attribute information removed can lead to a user embedding in which all sensitive attribute information is unlearned, demonstrating that LEGO can protect multiple sensitive attributes simultaneously.
While the bound is theoretically derived, its practical tightness and generalization across datasets and models are crucial for real-world applicability.
Since MI cannot be computed directly in our setting, we follow prior work and use the variational upper bound estimated by vCLUB as a proxy. 
We empirically evaluate the values of $P_1$ and $P_2$ across three datasets and three model architectures.

As shown in Table~\ref{tab:bound}, the gap between the theoretical bound and the estimated MI remains small across all settings, indicating that the bound is empirically tight.

\subsection{Sensitivity to $\epsilon$}
Parameter sensitivity results with respect to $\epsilon$ are shown in Figure~\ref{fig:hyperparam}.

\end{document}